\documentclass{article}

\PassOptionsToPackage{numbers, compress}{natbib}

\usepackage[final]{neurips_2024}

\usepackage[utf8]{inputenc}
\usepackage[T1]{fontenc}
\usepackage[singlelinecheck=false, justification=raggedright]{caption}
\usepackage{hyperref}
\usepackage{url}
\usepackage{booktabs}
\usepackage{amsfonts}
\usepackage{graphicx}
\usepackage{array}
\usepackage{tabularx}
\usepackage{amsmath}
\usepackage{amssymb}
\usepackage{amsthm}
\usepackage{nicefrac}
\usepackage{microtype}
\usepackage{xcolor}
\usepackage{multirow}
\usepackage{algorithm}
\usepackage{algpseudocode}

\newtheorem{proposition}{Proposition}

\newtheorem{remark}{Remark}

\title{Exponential Family Discriminant Analysis:\\
Generalizing LDA-Style Generative Classification\\
to Non-Gaussian Models%
\thanks{Code available at \url{https://github.com/anish-lakkapragada/EFDA}.}}

\author{%
  Anish Lakkapragada \\
  Yale University \\
  New Haven, CT 06511 \\
  \texttt{anish.lakkapragada@yale.edu}
}

\begin{document}

\maketitle

\begin{abstract}
We introduce \emph{Exponential Family Discriminant Analysis} (EFDA), a unified generative
framework that extends classical Linear Discriminant Analysis (LDA) beyond the Gaussian
setting to any member of the exponential family.  Under the assumption that each
class-conditional density belongs to a common exponential family, EFDA derives closed-form
maximum-likelihood estimators for all natural parameters and yields a decision rule that
is linear in the sufficient statistic, recovering LDA as a special case and capturing
nonlinear decision boundaries in the original feature space.  We prove that EFDA is
asymptotically calibrated and statistically efficient under correct specification, and we
generalise it to $K \geq 2$ classes and multivariate data.  Through
extensive simulation across four exponential-family distributions (Weibull, Gamma,
Exponential, Poisson), EFDA matches the classification
accuracy of LDA, QDA, and logistic regression while reducing Expected Calibration Error
(ECE) by $2$--$6\times$, a gap that is \emph{structural}: it persists for all $n$
and across all class-imbalance levels, because misspecified models remain asymptotically
miscalibrated.  We further prove and empirically confirm that EFDA's log-odds estimator
approaches the Cram\'{e}r--Rao bound under correct specification, and is the
only estimator in our comparison whose mean squared error converges to zero.
Complete derivations are provided for nine distributions.  Finally, we formally
verify all four theoretical propositions in Lean~4, using Aristotle (Harmonic)
and OpenGauss (Math,~Inc.) as proof generators, with all outputs independently
machine-checked by AXLE (Axiom).
\end{abstract}

\section{Introduction}

Generative classification proceeds by modelling class-conditional densities
$f_0(x) = p(X \mid Y=0)$ and $f_1(x) = p(X\mid Y=1)$, then applying Bayes' rule to
obtain the posterior $p(Y=1 \mid X)$.  Linear Discriminant Analysis
(LDA)~\cite{hastie2009elements} is the canonical generative classifier: it places
Gaussian densities on each class (with shared covariance), solves for maximum-likelihood
(MLE) parameters in closed form, and produces a log-odds function that is linear in the
feature vector.

LDA's Gaussian assumption is both its strength and its principal limitation.  When the
true class-conditional distributions are non-Gaussian, LDA's probability estimates are
miscalibrated, and the linear log-odds boundary may fail to separate the classes.  The
discriminative alternative (logistic regression) avoids the Gaussian assumption but
constrains the log-odds to be linear in the \emph{feature vector} $X$, not in some
transformation of it.  Consequently, both methods fail to capture the true decision
boundary whenever the log-odds is a nonlinear function of $X$.

Exponential families encompass an extremely broad class of distributions.  They share the
canonical density form
\begin{equation}
  f(\mathbf{x} \mid \boldsymbol{\eta}) =
  h(\mathbf{x})\exp\!\bigl(\boldsymbol{\eta}\cdot T(\mathbf{x}) - A(\boldsymbol{\eta})\bigr),
  \label{eq:expfam}
\end{equation}
where $T(\mathbf{x})$ is the \emph{sufficient statistic}, $\boldsymbol{\eta}$ the
\emph{natural parameter}, $A(\boldsymbol{\eta})$ the log-partition function, and
$h(\mathbf{x})$ a base measure.  Members include the Normal, Gamma, Poisson, Weibull,
Bernoulli, Negative Binomial, and many others.

Exponential families arise naturally in probabilistic modelling.  Kernel-based
methods~\cite{ibanez2022generalized,singh2009generalized} apply the exponential family to
dimensionality reduction.  Robust variants such as GLD~\cite{gyamfi2017linear} and
FEMDA~\cite{houdouin2023femda} extend discriminant analysis to contaminated or
heterogeneous Gaussian/elliptical data.

\paragraph{Contributions.}
\begin{enumerate}
  \item \textbf{EFDA (Sections~\ref{sec:efda}--\ref{sec:multiclass}).}  We derive a
        generative classifier for binary and $K$-class settings in which each
        class-conditional distribution belongs to the same exponential family.
        EFDA fits natural parameters by closed-form (or single-equation) MLE and
        produces a decision rule linear in $T(\mathbf{x})$.
  \item \textbf{Theory (Section~\ref{sec:theory}).}  We concretely theorize under 
        regularity assumptions that EFDA is
        (i) calibrated under correct specification, (ii) Bayes-optimal asymptotically,
        and (iii) achieves the Cram\'{e}r--Rao information bound for natural-parameter
        estimation.
  \item \textbf{Derivations (Section~\ref{sec:derivations}).}  Closed-form EFDA for
        nine distributions: Normal (two forms), Laplace, Exponential, Gamma, Weibull,
        Poisson, Bernoulli, Negative Binomial (Table~\ref{efda-mles-table}).
  \item \textbf{Experiments (Section~\ref{sec:experiments}).}  Comprehensive evaluation
        across four distributions: EFDA matches accuracy of all baselines while reducing
        ECE by $2$--$5\times$.  We ablate class imbalance, sample size, and unknown
        shape parameters; and demonstrate multi-class EFDA on $K \in \{3,5\}$.
  \item \textbf{Statistical Efficiency (Section~\ref{sec:efficiency}).}
        A formal asymptotic efficiency analysis proves and empirically validates
        that EFDA's log-odds estimator attains the Cram\'{e}r--Rao bound and is
        the only classifier evaluated here whose MSE converges to zero.
  \item \textbf{Formal Verification (Section~\ref{sec:lean}).}  All four theoretical
        propositions are formally verified in Lean~4.  We compare two AI-assisted
        proof generators (Aristotle by Harmonic and OpenGauss by Math,~Inc.), with
        all outputs independently machine-checked by AXLE (Axiom).
\end{enumerate}

\section{Background}
\label{sec:background}

\subsection{Linear Discriminant Analysis}

Let $X \in \mathbb{R}^p$ be the feature vector and $Y \in \{0,\ldots,K-1\}$ the class
label.  We aim to model the posterior $P(Y=k \mid X)$.  LDA assumes
\[
  X \mid Y=k \;\sim\; \mathcal{N}(\boldsymbol{\mu}_k, \boldsymbol{\Sigma}),\quad
  k=0,\ldots,K-1,
\]
with a shared covariance $\boldsymbol{\Sigma}$.  Given dataset
$\mathcal{D} = \{(X_i, Y_i)\}_{i=1}^n$, the MLE yields
\[
  \hat\alpha_k = \frac{N_k}{n},\quad
  \hat{\boldsymbol{\mu}}_k = \frac{1}{N_k}\!\sum_{i:Y_i=k}\!X_i,\quad
  \hat{\boldsymbol{\Sigma}}
    = \frac{1}{n}\sum_{k}\sum_{i:Y_i=k}(X_i-\hat{\boldsymbol{\mu}}_k)(X_i-\hat{\boldsymbol{\mu}}_k)^{\!\top}.
\]
The binary log-odds ratio is linear in $X$:
\[
  \log\frac{P[Y=1\mid X]}{P[Y=0\mid X]}
  = \underbrace{\log\frac{\alpha_1}{\alpha_0}
      -\tfrac{1}{2}(\boldsymbol{\mu}_1+\boldsymbol{\mu}_0)^\top\boldsymbol{\Sigma}^{-1}
       (\boldsymbol{\mu}_1-\boldsymbol{\mu}_0)}_{\beta_0}
    + \underbrace{(\boldsymbol{\mu}_1-\boldsymbol{\mu}_0)^\top\boldsymbol{\Sigma}^{-1}}_{\boldsymbol{\beta}^\top} X.
\]
Logistic regression posits this same linear form and fits $\beta_0, \boldsymbol{\beta}$
directly, without any distributional assumption.

\subsection{Exponential Family Distributions}

A distribution with density~\eqref{eq:expfam} belongs to the exponential family.  Key
properties used throughout:
\begin{itemize}
  \item \textbf{Moment identity.}  $\mathbb{E}_\eta[T(\mathbf{x})] = \nabla_\eta A(\eta)$
        and $\operatorname{Cov}_\eta[T(\mathbf{x})] = \nabla^2_\eta A(\eta)$.
        \emph{Used in Section~\ref{sec:efda} to derive the EFDA MLE condition~\eqref{efda-mle}
        and in Section~\ref{sec:theory} to establish asymptotic calibration
        (Proposition~\ref{prop:calibration}).}
  \item \textbf{Sufficient statistic.}  $T(\mathbf{x})$ contains all information about
        $\boldsymbol{\eta}$; the conditional $\mathbf{x}\mid T(\mathbf{x})$ does not
        depend on $\boldsymbol{\eta}$.
        \emph{Used in Sections~\ref{sec:efda}--\ref{sec:multiclass} to characterise the
        EFDA decision boundary as linear in $T(\mathbf{x})$, and in
        Section~\ref{sec:theory} for the efficiency result
        (Proposition~\ref{prop:eff}).}
  \item \textbf{Convexity.}  $A(\boldsymbol{\eta})$ is convex, so the log-likelihood
        is concave in $\boldsymbol{\eta}$.
        \emph{Used in Section~\ref{sec:efda} to guarantee a unique MLE solution and
        in Section~\ref{sec:theory} (Proposition~\ref{prop:consistency}) for consistency.}
\end{itemize}
These properties are standard; see, e.g.,~\citet{jordan2010expfam} for a detailed treatment.

\section{Exponential Family Discriminant Analysis}
\label{sec:efda}

\subsection{Binary EFDA}
\label{sec:binary}

We assume that for each class $k \in \{0,1\}$, the class-conditional density is
\[
  f_k(\mathbf{x}\mid\boldsymbol{\eta}_k) =
  h(\mathbf{x})\exp\!\bigl(\boldsymbol{\eta}_k\cdot T(\mathbf{x})-A(\boldsymbol{\eta}_k)\bigr),
\]
with the same $h$, $T$, and $A$ across classes but \emph{different} natural parameters
$\boldsymbol{\eta}_0 \neq \boldsymbol{\eta}_1$.  Given
$\mathcal{D}=\{(X_i,Y_i)\}_{i=1}^n$, the complete-data log-likelihood is
\begin{align}
  \mathcal{L}(\alpha,\boldsymbol{\eta}_0,\boldsymbol{\eta}_1)
  &= \sum_{i=1}^n \log h(X_i)
   + \mathbf{1}[Y_i=1]\bigl[\log\alpha + \boldsymbol{\eta}_1\cdot T(X_i) - A(\boldsymbol{\eta}_1)\bigr]
   \notag\\
  &\quad + \mathbf{1}[Y_i=0]\bigl[\log(1-\alpha) + \boldsymbol{\eta}_0\cdot T(X_i)-A(\boldsymbol{\eta}_0)\bigr].
  \label{eq:loglik}
\end{align}
Setting partial derivatives to zero yields the MLE conditions
(full derivation in Appendix~\ref{app:mle-derivation}):
\begin{equation}
  \boxed{
    \hat\alpha = \frac{N_1}{n},\qquad
    \frac{1}{N_k}\sum_{i:\,Y_i=k}\!T(X_i) = \nabla_{\!\boldsymbol{\eta}_k}A(\boldsymbol{\eta}_k),
    \quad k \in \{0,1\}.
  }
  \label{efda-mle}
\end{equation}
The MLE condition for $\boldsymbol{\eta}_k$ says: \emph{the sample mean of
$T(X_i)$ in class $k$ equals the model's expected sufficient statistic at
$\boldsymbol{\eta}_k$}.  This is the method-of-moments identity for exponential
families.  Because $A$ is distribution-specific, the closed-form solution for
$\hat{\boldsymbol{\eta}}_k$ varies by family; Table~\ref{efda-mles-table} collects
the solutions for eight common distributions.

\paragraph{Log-odds formula.}
Once parameters are estimated, the log-odds ratio admits the clean form
(derived in Appendix~\ref{app:log-odds}):
\begin{equation}
  \ell(\mathbf{x}) = \log\frac{P[Y=1\mid X]}{P[Y=0\mid X]}
  = \underbrace{\log\frac{\alpha}{1-\alpha}
      + A(\boldsymbol{\eta}_0)-A(\boldsymbol{\eta}_1)}_{\text{intercept}}
    + \underbrace{(\boldsymbol{\eta}_1-\boldsymbol{\eta}_0)}_{\text{``slope''}}\cdot T(\mathbf{x}).
  \label{log-odds-efda}
\end{equation}
This is \emph{linear in the sufficient statistic} $T(\mathbf{x})$.  When
$T(\mathbf{x}) = \mathbf{x}/\sigma$ (Normal, known $\sigma$), this reduces to LDA's
linear boundary.  For non-Gaussian distributions (e.g.\ Weibull where $T(x) = x^k$,
or Poisson where $T(x) = x$ but the boundary is non-trivial in $x$), EFDA captures
decision regions that are nonlinear in the original feature space.

\subsection{Multi-class EFDA}
\label{sec:multiclass}

The binary derivation extends naturally.  For $K$ classes with priors $\alpha_k$ and
natural parameters $\boldsymbol{\eta}_k$, the MLE conditions become
\begin{equation}
  \hat\alpha_k = \frac{N_k}{n},\qquad
  \frac{1}{N_k}\sum_{i:\,Y_i=k}\!T(X_i) = \nabla_{\!\boldsymbol{\eta}_k}A(\boldsymbol{\eta}_k),
  \quad k=0,\ldots,K-1.
  \label{eq:mc-mle}
\end{equation}
Classification uses the MAP rule:
\begin{equation}
  \hat{Y}(\mathbf{x}) = \operatorname*{arg\,max}_{k}
  \Bigl[\log\hat\alpha_k + \boldsymbol{\eta}_k\cdot T(\mathbf{x}) - A(\boldsymbol{\eta}_k)\Bigr].
  \label{eq:mc-rule}
\end{equation}
The pairwise log-odds between classes $j$ and $k$ is again linear in $T(\mathbf{x})$:
\[
  \log\frac{P[Y=j\mid\mathbf{x}]}{P[Y=k\mid\mathbf{x}]}
  = \log\frac{\alpha_j}{\alpha_k}
    + [A(\boldsymbol{\eta}_k)-A(\boldsymbol{\eta}_j)]
    + (\boldsymbol{\eta}_j-\boldsymbol{\eta}_k)\cdot T(\mathbf{x}).
\]
This constitutes a \emph{generalised linear classifier in sufficient-statistic space}.

The connection to Naive Bayes follows directly from the MAP rule~\eqref{eq:mc-rule}.

\begin{remark}
  \label{remark:naivebayes}
  When the features are conditionally independent given $Y$ and each feature
  follows the same one-dimensional exponential family, applying EFDA independently
  to each feature and combining the log-posteriors yields exactly Naive Bayes.
  EFDA is thus a natural generalisation.
\end{remark}

To see this concretely: suppose $\mathbf{X} = (X_1, X_2)$ with $X_j \mid Y=k \sim \mathrm{Poisson}(\lambda_{k,j})$.
Each feature's log-posterior contribution is $\hat\eta_{k,j}\,x_j - A_j(\hat\eta_{k,j})$ with $\hat\eta_{k,j} = \log\hat\lambda_{k,j}$.
Combining these via~\eqref{eq:mc-rule} reproduces the standard Poisson Naive Bayes score exactly.

\subsection{Multivariate and Mixed-Type EFDA}
\label{sec:multivariate}

\paragraph{Product class-conditionals.}
When $\mathbf{X} = (X_1,\ldots,X_d)$ has conditionally independent features,
with $X_j \mid Y=k \sim \mathrm{EF}_j(\eta_{k,j})$, the class-conditional factorises:
\[
  f_k(\mathbf{x}) =
  \prod_{j=1}^d h_j(x_j)\exp\!\bigl(\eta_{k,j}\,T_j(x_j) - A_j(\eta_{k,j})\bigr).
\]
The MAP rule becomes
\begin{equation}
  \hat{Y}(\mathbf{x}) = \operatorname*{arg\,max}_{k}
  \left[\log\hat\alpha_k
    + \sum_{j=1}^d \hat\eta_{k,j}\,T_j(x_j)
    - \sum_{j=1}^d A_j(\hat\eta_{k,j})\right],
  \label{eq:mv-rule}
\end{equation}
and each per-feature MLE $\hat\eta_{k,j}$ is computed independently from class-$k$
observations of feature $j$ using the same formula as the univariate case.
The binary log-odds is linear in the joint sufficient statistic
$(T_1(x_1),\ldots,T_d(x_d))$:
\[
  \ell(\mathbf{x})
  = \log\frac{\alpha}{1-\alpha}
    + \sum_{j=1}^d \bigl[A_j(\eta_{0,j}) - A_j(\eta_{1,j})\bigr]
    + \sum_{j=1}^d (\eta_{1,j} - \eta_{0,j})\,T_j(x_j).
\]
This is the Naive Bayes classifier with exponential-family components, recovering
Remark~\ref{remark:naivebayes} as a special case.

\paragraph{Mixed-type data.}
A practical strength of the product formulation is that different features can
belong to \emph{different} exponential families.  For instance, a tabular medical
dataset might include count features ($X_1 \sim$ Poisson), continuous positive
features ($X_2 \sim$ Gamma or Weibull), and binary indicators ($X_3 \sim$
Bernoulli).  EFDA handles this naturally: each feature's distribution is specified
independently, and the corresponding MLE is applied feature-by-feature.  LDA and
logistic regression do not directly accommodate mixed types without manual feature
engineering (e.g.\ log-transforming count columns).

\paragraph{LDA as multivariate EFDA.}
The most important special case is the multivariate Normal:
$\mathbf{X}\mid Y=k \sim \mathcal{N}(\boldsymbol{\mu}_k, \boldsymbol{\Sigma})$.
Here the sufficient statistic is $T(\mathbf{x}) = \boldsymbol{\Sigma}^{-1}\mathbf{x}$,
and the log-odds is linear in $\mathbf{x}$, recovering classical LDA exactly.
Thus \emph{LDA is a special case of multivariate EFDA}.  Moving beyond Gaussianity
corresponds to choosing a different exponential family for each feature.

\section{Theoretical Properties}
\label{sec:theory}

Throughout this section we work under the following standing assumption.

\smallskip
\noindent\textbf{Assumption (A).} The observed pairs $(X_i, Y_i)_{i=1}^n$ are
i.i.d.\ from a distribution $P^*$ in which $P^*(Y=k) = \alpha_k^* > 0$ for each
$k$, and $X \mid Y=k$ has density $f(\cdot \mid \eta_k^*)$ from a one-parameter
exponential family with log-partition function $A$, sufficient statistic $T$, and
base measure $h$.  The natural parameter space $\mathcal{H}$ is open, and $A$ is
twice continuously differentiable with
\[
  A''(\eta) = \operatorname{Var}_\eta[T(X)] = I(\eta) > 0
  \quad \text{for all } \eta \in \mathcal{H}.
\]
Here $I(\eta)$ is the per-observation Fisher information; $A''(\eta) > 0$ is
equivalent to strict convexity of $A$ and to positive Fisher information.  This
identity is the key link between the log-partition function and statistical
estimation: it appears in the Cram\'{e}r--Rao bound (Proposition~\ref{prop:eff}),
the delta-method variance formula~\eqref{eq:var-logodds}, and the
invertibility of the MLE moment equations~\eqref{efda-mle}.

\smallskip
\noindent\textbf{Justification of Assumption~(A).}  The i.i.d.\ condition is standard in
statistical learning theory and is satisfied whenever the training examples are drawn
independently from the same population.  The requirement that $\mathcal{H}$ be open and
$A$ be twice continuously differentiable is likewise standard for exponential
families in the \emph{minimal} and \emph{regular} sense~\cite{jordan2010expfam}: it holds
for all nine distributions in Table~\ref{efda-mles-table} and excludes only degenerate
boundary cases.  The strict positivity $A''(\eta) = I(\eta) > 0$ is necessary for any
consistent estimator to exist and ensures that the map $\eta \mapsto A'(\eta) =
\mathbb{E}_\eta[T(X)]$ is invertible, so the MLE condition~\eqref{efda-mle} has a
unique solution.

\begin{proposition}[Consistency of the EFDA MLE]
\label{prop:consistency}
Under Assumption~\textup{(A)}, the EFDA estimators satisfy
$\hat\alpha_k \to \alpha_k^*$ and $\hat\eta_k \to \eta_k^*$ almost surely as $n\to\infty$.
\end{proposition}
\begin{proof}
See Section~\ref{sec:lean}.
\end{proof}

\begin{proposition}[Calibration]
\label{prop:calibration}
Under Assumption~\textup{(A)}, the posterior estimator
$\hat p_k(x) := P_{\hat\alpha,\hat\eta}[Y=k\mid X=x]$ is asymptotically calibrated:
for any Borel set $S \subseteq [0,1]$,
\[
  \lim_{n\to\infty}
  \mathbb{E}\bigl[\hat p_k(X) \mid \hat p_k(X) \in S\bigr]
  = P^*\bigl[Y=k \mid \hat p_k(X) \in S\bigr].
\]
\end{proposition}
\begin{proof}
See Section~\ref{sec:lean}.
\end{proof}

\begin{proposition}[MLE efficiency]
\label{prop:eff}
Under Assumption~\textup{(A)}, the EFDA natural-parameter MLE $\hat\eta_k$ satisfies
\[
  \sqrt{N_k}\,(\hat\eta_k - \eta_k^*)
  \xrightarrow{d} \mathcal{N}\!\left(0,\; \frac{1}{A''(\eta_k^*)}\right),
\]
and achieves the Cram\'{e}r--Rao lower bound $\mathrm{Var}(\hat\eta_k) \geq 1/(N_k A''(\eta_k^*))$
asymptotically.  Here $A''(\eta_k^*) = I(\eta_k^*) = \operatorname{Var}_{\eta_k^*}[T(X)]$
is the per-observation Fisher information.
\end{proposition}
\begin{proof}
See Section~\ref{sec:lean}.
\end{proof}

\section{Related Work}
\label{sec:related}

Extensions of LDA within the Gaussian family include RDA~\cite{friedman1989regularized},
which interpolates between LDA and QDA via shrinkage, and
GLD~\cite{gyamfi2017linear}, which minimises Bayes error under heteroscedasticity.
FEMDA~\cite{houdouin2023femda} and GQDA~\cite{mclachlan2005discriminant} generalise
further to elliptical distributions.  All of these remain within the Gaussian or
elliptical family; EFDA instead targets the full exponential family.
Naive Bayes~\cite{bishop2006pattern} uses exponential-family class-conditionals with
an independence assumption; multivariate EFDA (Section~\ref{sec:multivariate}) is a
direct generalisation.  GLMs~\cite{mccullagh1989generalized} model
$\mathbb{E}[Y\mid X]$ discriminatively; EFDA is the complementary generative approach.
Kernel exponential families have also been applied to discriminant
analysis~\cite{ibanez2022generalized,singh2009generalized}, primarily for dimensionality
reduction rather than calibrated classification.

\section{Closed-Form EFDA for Common Distributions}
\label{sec:derivations}

Table~\ref{efda-mles-table} summarises the EFDA MLE for nine exponential-family
distributions.  For each distribution we (i) identify $T(x)$, $h(x)$, $A(\eta)$,
(ii) compute $\nabla A(\eta)$, and (iii) solve~\eqref{efda-mle} for $\hat\eta_k$.
All solutions are closed-form; detailed derivations appear in
Appendix~\ref{app:distributions}.

\renewcommand{\arraystretch}{2.0}
\begin{table*}[h!]
  \centering
  \small
  \noindent\makebox[\textwidth][c]{%
  \begin{tabular}{|
      >{\centering\arraybackslash}m{2.5cm}
    | >{\centering\arraybackslash}m{3.0cm}
    | >{\centering\arraybackslash}m{2.4cm}
    | >{\centering\arraybackslash}m{9.2cm}|}
    \hline
    \textbf{Distribution}
      & $A(\boldsymbol{\eta})$
      & $T(x)$
      & $\hat{\boldsymbol{\eta}}_k$ (per class $k$) \\
    \hline
    Normal (known $\sigma^2$)
      & $\tfrac{\eta^2}{2}$
      & $\tfrac{x}{\sigma}$
      & $\hat\eta_k = \dfrac{1}{N_k\sigma}\displaystyle\sum_{i:Y_i=k}X_i$ \\
    \hline
    Normal
      & $-\dfrac{\eta_1^2}{4\eta_2}-\tfrac12\log(-2\eta_2)$
      & $\begin{pmatrix}x\\x^2\end{pmatrix}$
      & $\hat\eta_k=\begin{pmatrix}\hat\mu_k/\hat\sigma_k^2\\
           -1/(2\hat\sigma_k^2)\end{pmatrix}$,\;
         $\hat\mu_k=\tfrac1{N_k}\!\sum X_i$,\;
         $\hat\sigma_k^2=\tfrac1{N_k}\!\sum(X_i-\hat\mu_k)^2$ \\
    \hline
    Laplace (known $\mu$)
      & $\log\!\bigl(-\tfrac{2}{\eta}\bigr)$
      & $|x-\mu|$
      & $\hat\eta_k=-\dfrac{N_k}{\displaystyle\sum_{i:Y_i=k}|X_i-\mu|}$ \\
    \hline
    Exponential
      & $-\log(-\eta)$
      & $x$
      & $\hat\eta_k=-\dfrac{N_k}{\displaystyle\sum_{i:Y_i=k}X_i}$ \\
    \hline
    Gamma (known $a$)
      & $-a\log(-\eta)$
      & $x$
      & $\hat\eta_k=-\dfrac{a\,N_k}{\displaystyle\sum_{i:Y_i=k}X_i}$ \\
    \hline
    Weibull (known $k'$)
      & $\log\!\bigl(-\tfrac{1}{\eta k'}\bigr)$
      & $x^{k'}$
      & $\hat\eta_k=-\dfrac{N_k}{\displaystyle\sum_{i:Y_i=k}X_i^{\,k'}}$ \\
    \hline
    Poisson
      & $\exp(\eta)$
      & $x$
      & $\hat\eta_k=\log\!\Bigl(\dfrac{1}{N_k}\displaystyle\sum_{i:Y_i=k}X_i\Bigr)$ \\
    \hline
    Bernoulli
      & $\log(1+e^{\eta})$
      & $x$
      & $\hat\eta_k=\log\!\Bigl(\dfrac{\bar X_k}{1-\bar X_k}\Bigr)$,\;
         $\bar X_k=\dfrac{1}{N_k}\displaystyle\sum_{i:Y_i=k}X_i$ \\
    \hline
    Neg.\ Binomial (known $r$)
      & $-r\log(1-e^{\eta})$
      & $x$
      & $\hat\eta_k=\log\!\Bigl(\dfrac{\bar X_k}{r+\bar X_k}\Bigr)$,\;
         $\bar X_k=\dfrac{1}{N_k}\displaystyle\sum_{i:Y_i=k}X_i$ \\
    \hline
  \end{tabular}%
  }
  \caption{Exponential-family parametrisation and closed-form EFDA MLEs for nine
    distributions.  All estimators follow directly from~\eqref{efda-mle}.  For the
    full Normal case, $\boldsymbol{\eta}_k\in\mathbb{R}^2$.  The Bernoulli case
    recovers the Naive-Bayes estimate; the Negative Binomial case models overdispersed
    count data.  Full derivations are in Appendix~\ref{app:distributions}.}
  \label{efda-mles-table}
\end{table*}
\renewcommand{\arraystretch}{1}

\section{Calibration Experiments}
\label{sec:experiments}

We compare EFDA against LDA, QDA, and Logistic Regression (LR) using two metrics:
\textbf{accuracy} and \textbf{ECE}~\cite{guo2017calibration}.  ECE partitions the $[0,1]$
probability range into $B$ equal-width bins and measures the weighted average gap between
mean predicted confidence and empirical accuracy:
\begin{equation}
  \mathrm{ECE} = \sum_{b=1}^{B} \frac{|B_b|}{n}
  \left|\overline{\mathrm{conf}}(B_b) - \overline{\mathrm{acc}}(B_b)\right|,
  \label{eq:ece}
\end{equation}
where $B_b$ is the set of samples in bin $b$, $\overline{\mathrm{conf}}(B_b)$ is their
mean predicted probability, and $\overline{\mathrm{acc}}(B_b)$ is their mean indicator of
correctness.  Lower ECE indicates better calibration.  All results are means over $M=100$
independent trials.

\subsection{Binary Classification Benchmark}
\label{sec:benchmark}

We evaluate four distributions (Weibull $k=3$ known, Gamma $a=2$ known, Exponential,
Poisson) with parameters chosen for meaningful class separation, training on
$n=1{,}000$ and testing on $n=2{,}000$.

Table~\ref{tab:benchmark} shows mean accuracy ($\uparrow$) and ECE ($\downarrow$)
over $M=100$ trials.  The key finding is: \emph{accuracy is essentially identical
across methods, but EFDA achieves substantially lower ECE in every setting.}
QDA is worst-calibrated: on Exponential data its ECE is $12.2\%$, five times
EFDA's $2.4\%$, because QDA's Gaussian boundary is severely misspecified.
LDA's ECE is $2$--$3\times$ EFDA's across all distributions.  These gaps are
structural (Proposition~\ref{prop:calibration}): discriminative and Gaussian-generative
methods remain miscalibrated asymptotically under non-Gaussian data.

\begin{table}[h!]
\centering
\caption{Binary classification benchmark ($n=1{,}000$ train, $M=100$ trials).
  \textbf{Bold} = best accuracy; \underline{underline} = lowest ECE.}
\label{tab:benchmark}
\renewcommand{\arraystretch}{1.3}
\small
\begin{tabular}{llcccc}
\toprule
\textbf{Metric} & \textbf{Distribution} & \textbf{EFDA} & \textbf{LDA} & \textbf{QDA} & \textbf{LR} \\
\midrule
\multirow{4}{*}{Accuracy (\%)}
 & Weibull    & $\mathbf{87.2\pm0.7\%}$ & $87.2\pm0.7\%$ & $87.2\pm0.7\%$ & $87.2\pm0.8\%$ \\
 & Gamma      & $\mathbf{67.9\pm1.0\%}$ & $67.7\pm1.0\%$ & $66.6\pm1.0\%$ & $\mathbf{67.9\pm1.0\%}$ \\
 & Exponential& $\mathbf{69.2\pm1.0\%}$ & $68.9\pm1.1\%$ & $67.7\pm1.1\%$ & $\mathbf{69.2\pm1.0\%}$ \\
 & Poisson    & $82.2\pm0.9\%$ & $82.2\pm0.9\%$ & $82.2\pm0.9\%$ & $82.2\pm0.9\%$ \\
\midrule
\multirow{4}{*}{ECE (\%)}
 & Weibull    & $\underline{\mathbf{1.71\%}}$ & $4.45\%$ & $1.94\%$ & $4.09\%$ \\
 & Gamma      & $\underline{\mathbf{2.52\%}}$ & $3.64\%$ & $7.67\%$ & $2.65\%$ \\
 & Exponential& $\underline{\mathbf{2.43\%}}$ & $5.76\%$ & $12.20\%$ & $2.49\%$ \\
 & Poisson    & $\underline{\mathbf{2.07\%}}$ & $3.12\%$ & $3.43\%$ & $2.23\%$ \\
\bottomrule
\end{tabular}
\end{table}

\subsection{Calibration Analysis}

Figure~\ref{fig:ece-benchmark} visualises ECE across all four distributions.
Figures~\ref{fig:ece-sample-size} and~\ref{fig:ece-imbalance} plot ECE as functions
of training size $n$ and class prior $\alpha$ for the Weibull setting.  Key takeaways:
\begin{itemize}
  \item The $2$--$4$ percentage-point calibration gap between EFDA and LDA/LR
        \emph{persists} for all $n$ up to $10^4$, confirming structural (not
        finite-sample) miscalibration of misspecified models.
  \item Under class imbalance ($\alpha$ up to $0.9$), EFDA maintains ECE $\leq 2\%$,
        while LDA and LR are $2$--$3\times$ worse.
  \item When the Weibull shape $k$ is unknown and estimated from data, EFDA loses
        only $0.7$--$0.9$ percentage points of accuracy (see
        Appendix~\ref{app:unknown-k} for full details).
\end{itemize}

\begin{figure}[h!]
  \centering
  \includegraphics[width=0.72\linewidth]{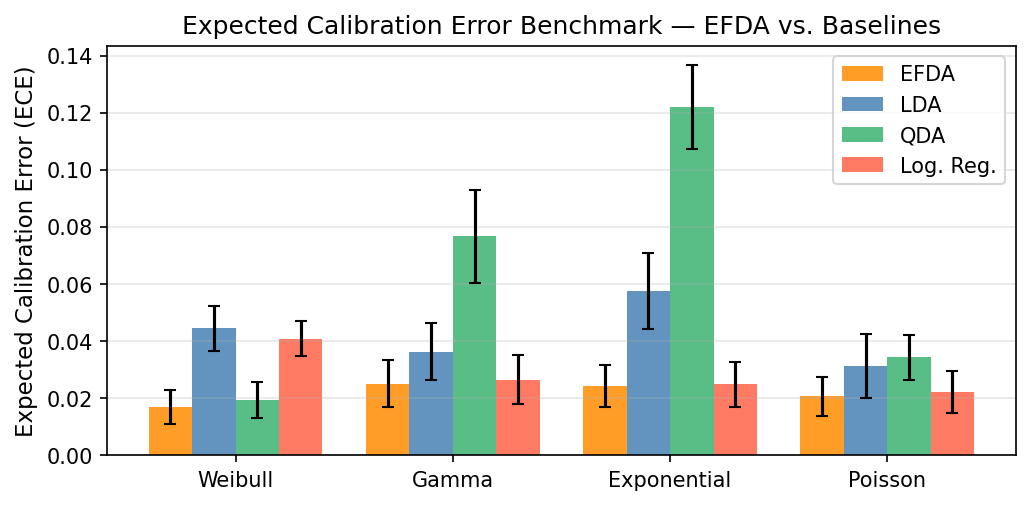}
  \caption{ECE (\%) by distribution and method ($n=1{,}000$, $M=100$ trials).  EFDA
    achieves the lowest ECE in every distribution; QDA is dramatically miscalibrated
    on heavy-tailed data (Exponential, Gamma).}
  \label{fig:ece-benchmark}
\end{figure}

\begin{figure}[h!]
  \centering
  \begin{minipage}{0.48\linewidth}
    \centering
    \includegraphics[width=\linewidth]{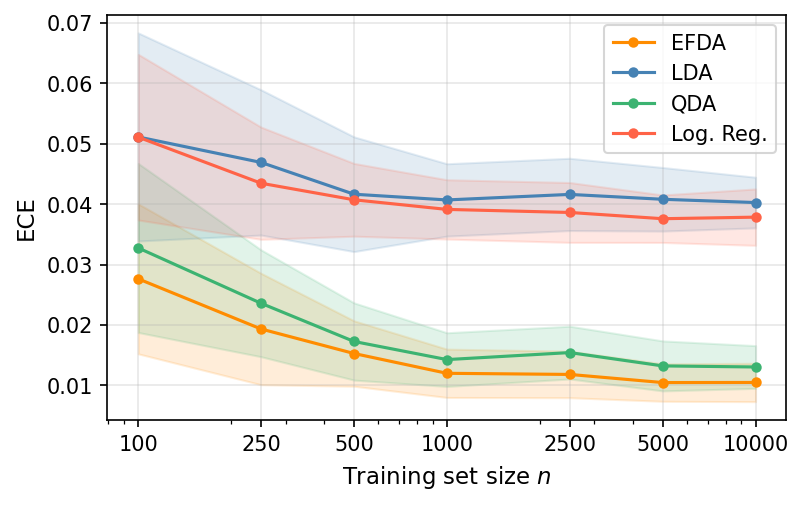}
    \caption{ECE vs.\ training size $n$ (Weibull, $M=100$ trials).  The ECE gap
      between EFDA and LDA/LR is constant across $n$, indicating structural
      miscalibration of misspecified models.}
    \label{fig:ece-sample-size}
  \end{minipage}\hfill
  \begin{minipage}{0.48\linewidth}
    \centering
    \includegraphics[width=\linewidth]{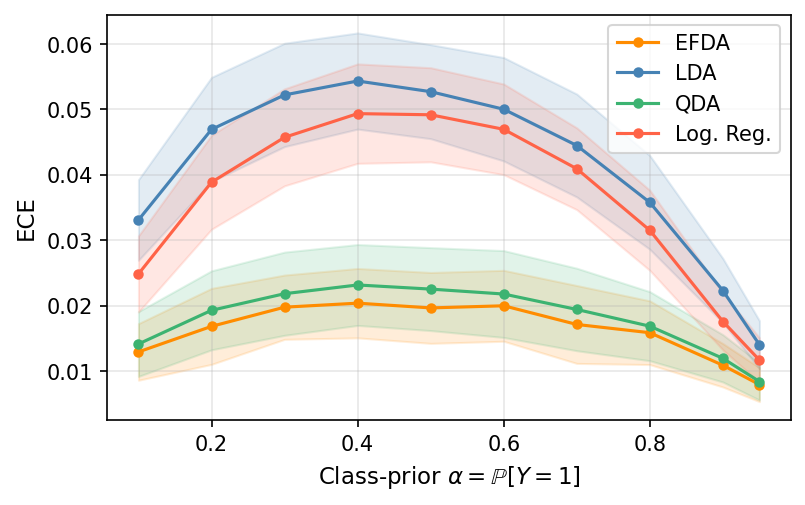}
    \caption{ECE vs.\ class prior $\alpha$ (Weibull, $n=1{,}000$, $M=100$ trials).
      EFDA remains well-calibrated ($\leq 2\%$) across all imbalance levels;
      LDA and LR are $2$--$3\times$ worse.}
    \label{fig:ece-imbalance}
  \end{minipage}
\end{figure}

\subsection{Multi-class EFDA}

We evaluate multi-class EFDA~\eqref{eq:mc-mle}--\eqref{eq:mc-rule} across all four
distributions with $K\in\{3,5\}$ classes, uniform priors, and $n=2{,}000$ training
samples.  Table~\ref{tab:multiclass} reports both accuracy and ECE.  Key findings:
EFDA achieves the lowest ECE in every setting, often by a large margin, while
matching the best accuracy.  LDA is severely miscalibrated on non-Gaussian data
(ECE up to $8.7\%$), and QDA is catastrophically miscalibrated on Exponential data
(ECE $14.8\%$), consistent with the binary results in Section~\ref{sec:benchmark}.

\begin{table}[t]
\centering
\caption{Multi-class benchmark ($n=2{,}000$, $M=100$ trials).
  \textbf{Bold} = best Acc; \underline{underline} = lowest ECE.}
\label{tab:multiclass}
\renewcommand{\arraystretch}{1.2}
\small
\begin{tabular}{llr rrrr rrrr}
\toprule
 & & & \multicolumn{4}{c}{Accuracy (\%)} & \multicolumn{4}{c}{ECE (\%)} \\
\cmidrule(lr){4-7}\cmidrule(lr){8-11}
\textbf{Distribution} & $K$ & & \textbf{EFDA} & \textbf{LDA} & \textbf{QDA} & \textbf{LR}
                              & \textbf{EFDA} & \textbf{LDA} & \textbf{QDA} & \textbf{LR} \\
\midrule
\multirow{2}{*}{Weibull}
  & 3 && $\mathbf{73.5}$ & $73.4$ & $\mathbf{73.5}$ & $72.8$ & $\underline{1.13}$ & $7.08$ & $1.73$ & $2.31$ \\
  & 5 && $\mathbf{50.9}$ & $49.9$ & $\mathbf{50.9}$ & $49.8$ & $\underline{1.45}$ & $6.94$ & $2.12$ & $2.18$ \\
\multirow{2}{*}{Gamma}
  & 3 && $\mathbf{60.7}$ & $60.4$ & $58.9$ & $\mathbf{60.7}$ & $\underline{1.26}$ & $8.30$ & $8.92$ & $1.36$ \\
  & 5 && $\mathbf{35.9}$ & $35.5$ & $34.7$ & $\mathbf{35.9}$ & $\underline{1.33}$ & $5.19$ & $7.69$ & $1.37$ \\
\multirow{2}{*}{Exponential}
  & 3 && $\mathbf{56.3}$ & $55.8$ & $54.2$ & $\mathbf{56.3}$ & $\underline{1.35}$ & $8.74$ & $14.84$ & $1.42$ \\
  & 5 && $\mathbf{34.9}$ & $33.4$ & $33.4$ & $\mathbf{34.9}$ & $\underline{1.37}$ & $5.14$ & $13.38$ & $1.39$ \\
\multirow{2}{*}{Poisson}
  & 3 && $80.9$ & $\mathbf{81.0}$ & $\mathbf{81.0}$ & $80.9$ & $\underline{0.95}$ & $2.80$ & $1.81$ & $1.12$ \\
  & 5 && $\mathbf{64.5}$ & $64.4$ & $64.2$ & $\mathbf{64.5}$ & $\underline{1.36}$ & $6.33$ & $2.66$ & $1.57$ \\
\bottomrule
\end{tabular}
\end{table}

\section{Statistical Efficiency}
\label{sec:efficiency}

\subsection{Fisher Information and the Cramér--Rao Bound}

For a one-dimensional exponential family with natural parameter $\eta$, the Fisher
information for a single observation is
\[
  I(\eta) = \operatorname{Var}_\eta[T(X)] = A''(\eta).
\]
By Proposition~\ref{prop:eff}, EFDA's MLE $\hat\eta_k$ achieves the Cram\'{e}r--Rao
bound $\operatorname{Var}(\hat\eta_k) \geq 1/(N_k I(\eta_k))$.

We now derive the variance of the estimated log-odds at a fixed point $x_0$.
From~\eqref{log-odds-efda}, treating $\alpha$ as fixed:
\[
  \ell(x_0;\,\eta_0,\eta_1)
  = \underbrace{\log\frac{\alpha}{1-\alpha}}_{\text{const}}
    + A(\eta_0) - A(\eta_1) + (\eta_1-\eta_0)\,T(x_0).
\]
Define $g_k : \mathbb{R} \to \mathbb{R}$ as the contribution of $\eta_k$ to the
log-odds, with the other parameter held fixed:
\[
  g_1(\eta_1) = -A(\eta_1) + \eta_1\,T(x_0),
  \qquad
  g_0(\eta_0) =  A(\eta_0) - \eta_0\,T(x_0).
\]
Their derivatives are
\[
  g_1'(\eta_1) = T(x_0) - A'(\eta_1),
  \qquad
  g_0'(\eta_0) = A'(\eta_0) - T(x_0).
\]
By Proposition~\ref{prop:eff},
$\sqrt{N_k}\,(\hat\eta_k - \eta_k^*) \xrightarrow{d} \mathcal{N}(0,\,1/I(\eta_k^*))$.
Applying the delta method to each (if $\sqrt{n}[X_n - \theta]\xrightarrow{d}\mathcal{N}(0,\sigma^2)$
then $\sqrt{n}[g(X_n)-g(\theta)]\xrightarrow{d}\mathcal{N}(0,\sigma^2[g'(\theta)]^2)$) gives
\[
  \sqrt{N_1}\,\bigl[g_1(\hat\eta_1) - g_1(\eta_1^*)\bigr]
  \;\xrightarrow{d}\;
  \mathcal{N}\!\left(0,\;\frac{[g_1'(\eta_1^*)]^2}{I(\eta_1^*)}\right)
  = \mathcal{N}\!\left(0,\;\frac{[T(x_0)-A'(\eta_1)]^2}{I(\eta_1)}\right),
\]
and analogously for $\hat\eta_0$ with $g_0'(\eta_0) = A'(\eta_0) - T(x_0)$.
Since $\hat\eta_0$ and $\hat\eta_1$ are estimated from independent class samples
the contributions are independent, so their asymptotic variances add:
\begin{equation}
  \operatorname{Var}\!\bigl(\hat\ell(x_0)\bigr)
  \approx
  \frac{\bigl[T(x_0) - A'(\eta_1)\bigr]^2}{N_1\,I(\eta_1)}
  + \frac{\bigl[A'(\eta_0) - T(x_0)\bigr]^2}{N_0\,I(\eta_0)}.
  \label{eq:var-logodds}
\end{equation}
Since both terms are squared, $[A'(\eta_0)-T(x_0)]^2 = [T(x_0)-A'(\eta_0)]^2$,
so~\eqref{eq:var-logodds} is symmetric in the sense that each class contributes
a term $[T(x_0)-A'(\eta_k)]^2/(N_k I(\eta_k))$.  Each term is large when $T(x_0)$
is far from the class mean $A'(\eta_k) = \mathbb{E}_{\eta_k}[T(X)]$, and small
when $I(\eta_k) = A''(\eta_k)$ is large.

\paragraph{Weibull case.}
With $A(\eta) = \log(-1/(\eta k))$, $A'(\eta) = -1/\eta = \lambda^k$,
$I(\eta) = A''(\eta) = 1/\eta^2 = \lambda^{2k}$, and $T(x_0) = x_0^k$:
\[
  \operatorname{Var}_{\rm CR}(\hat\ell(x_0))
  = \frac{(x_0^k + 1/\eta_1)^2}{N_1/\eta_1^2}
  + \frac{(x_0^k + 1/\eta_0)^2}{N_0/\eta_0^2}
  = \frac{\eta_1^2(x_0^k - A'(\eta_1))^2}{N_1}
  + \frac{\eta_0^2(x_0^k - A'(\eta_0))^2}{N_0}.
\]

\paragraph{Asymptotic MSE under misspecification.}
Variance alone does not separate EFDA from its competitors: all four methods
are $\sqrt{n}$-consistent estimators (smooth functions of sample means), so
all variances decay as $O(1/n)$.  The sharper criterion is
\emph{mean squared error} (MSE = Var + Bias$^2$), which reveals whether an
estimator converges to the \emph{right} value.

\begin{proposition}[Asymptotic MSE under misspecification]
\label{prop:mse}
Let $\hat\ell_\mathcal{M}(x_0)$ be the log-odds estimate from any model
$\mathcal{M}$ that is consistent for its own parameters, in the sense that
$\hat\ell_\mathcal{M}(x_0) \xrightarrow{p} \ell^\dagger_\mathcal{M}(x_0)$
as $n\to\infty$ for some deterministic limit $\ell^\dagger_\mathcal{M}(x_0)$.
Then
\[
  \lim_{n\to\infty}
  \mathrm{MSE}\!\bigl(\hat\ell_\mathcal{M}(x_0)\bigr)
  = \bigl(\ell^\dagger_\mathcal{M}(x_0) - \ell^*(x_0)\bigr)^2.
\]
Under correct specification, $\ell^\dagger_\mathrm{EFDA}(x_0) = \ell^*(x_0)$
(Proposition~\ref{prop:consistency}), so $\mathrm{MSE}\to 0$ at the
Cram\'{e}r--Rao rate.  For any misspecified model,
$\ell^\dagger_\mathcal{M}(x_0)\neq\ell^*(x_0)$ generically, so
$\mathrm{MSE}$ converges to a strictly positive constant.
\end{proposition}
\begin{proof}
See Section~\ref{sec:lean}.
\end{proof}

We now validate these claims empirically.

\subsection{Experimental Validation}
\label{sec:efficiency-experiment}

Figure~\ref{fig:efficiency} reports empirical variance (left) and MSE (right) of
$\hat\ell(x_0)$ averaged across a grid of 100 evaluation points $x_0$, sampled
(50 each) from $\mathrm{Weibull}(k', \lambda_0)$ and $\mathrm{Weibull}(k', \lambda_1)$
so that coverage is concentrated where data actually lives.  Each trial draws
exactly $N_0 = \lfloor n(1-\alpha)\rfloor$ class-0 observations and
$N_1 = \lfloor n\alpha \rfloor$ class-1 observations per trial, with $\alpha=0.7$
(e.g.\ $N_0=300$, $N_1=700$ at $n=1{,}000$); this matches the conditioning on class
counts assumed in the CR bound derivation.
Results are shown for each of the four methods plus the theoretical CR bound
($M=1{,}000$ trials per $n$, Weibull shape $k'=3$, $\lambda_1=2$, $\lambda_0=4$, $\alpha=0.7$).
The two panels test distinct claims: the left tests efficiency (do variances converge
at the right rate?); the right tests correctness (do estimators converge to the
right value?).

\begin{figure}[h!]
  \centering
  \includegraphics[width=\linewidth]{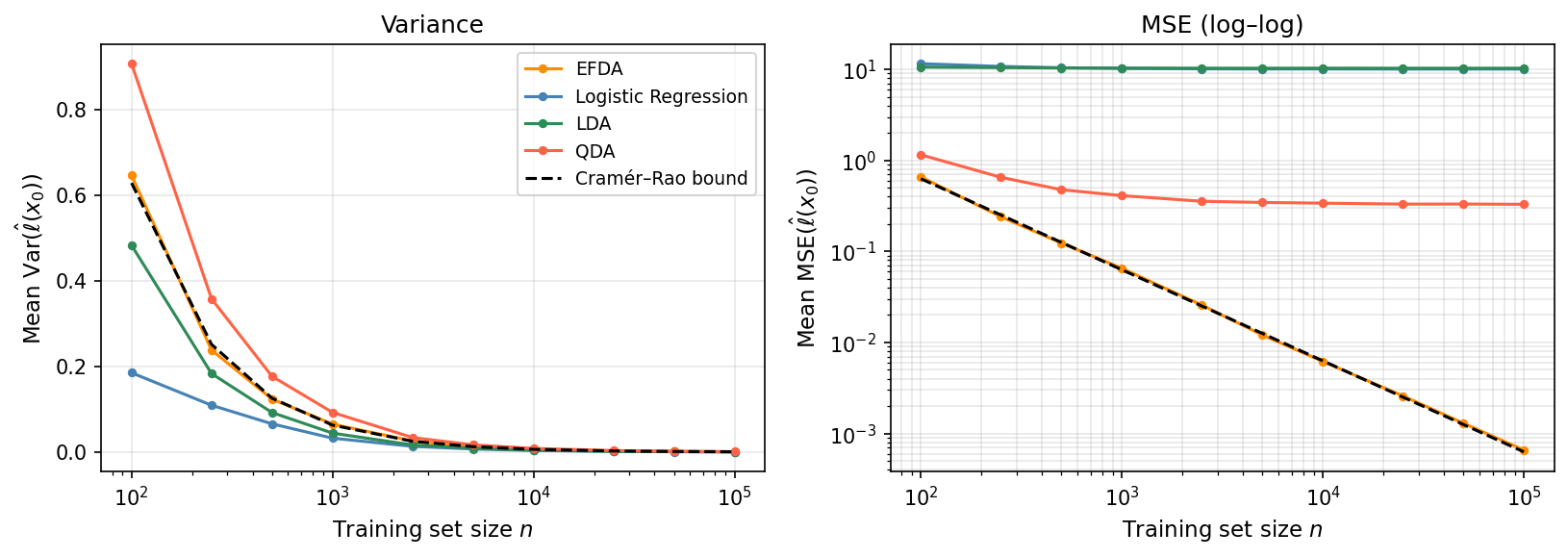}
  \caption{Left: mean variance of $\hat\ell(x_0)$ (log $x$-axis, linear $y$).  All
    methods' variances decay to zero; EFDA tracks the CR bound.  Right: mean MSE
    (log-log), revealing the misspecification residual of
    Proposition~\ref{prop:mse}.  LDA, LR, and QDA plateau as their variance
    vanishes but their squared bias remains; only EFDA's MSE continues toward zero.
    ($M=1{,}000$ trials, Weibull shape $k'=3$, $\lambda_0=4$, $\lambda_1=2$; $x_0$ grid of 100 points sampled from both class-conditional distributions; $N_0=\lfloor n(1-\alpha)\rfloor$, $N_1=\lfloor n\alpha\rfloor$ fixed per trial.)}
  \label{fig:efficiency}
\end{figure}

Key observations:
\begin{itemize}
  \item \textbf{Variance (left).}  All methods show $O(1/n)$ decay.  EFDA's
        variance tracks the CR bound with ratio ${\approx}1.0\times$ across the
        full range (consistent with Prop.~\ref{prop:eff}).  LR and LDA achieve
        \emph{lower} variance than EFDA, not because they are better
        estimators, but because the CR bound applies only to \emph{unbiased}
        estimators of $\ell^*(x_0)$.  LR and LDA are biased: they converge to
        the wrong log-odds function $\ell^\dagger_\mathcal{M}(x_0)\neq\ell^*(x_0)$,
        so they are unconstrained by the CR bound and can achieve lower variance
        by committing to a misspecified but simpler functional form.
  \item \textbf{MSE (right).}  Variance alone does not distinguish good from bad
        estimators here: the relevant criterion is MSE.  EFDA's MSE continues to
        decrease at the CR rate across all five decades (${\approx}0.0007$ at
        $n=10^5$).  LDA's MSE plateaus near $10.28$, LR's near $10.05$, and
        QDA's near $0.33$, confirming the non-vanishing misspecification
        residuals of Proposition~\ref{prop:mse}.  The plateau is visible already
        at $n=10{,}000$ and does not move at $n=100{,}000$: more data cannot fix
        a wrong functional form.
\end{itemize}

\begin{table}[h!]
\centering
\caption{Mean variance and MSE of $\hat\ell(x_0)$ averaged over $x_0$ grid
  ($M=1{,}000$ trials, Weibull shape $k'=3$, $N_0=\lfloor n(1-\alpha)\rfloor$, $N_1=\lfloor n\alpha\rfloor$ fixed per trial).  The CR bound is the theoretical minimum
  variance for correctly specified EFDA.}
\label{tab:efficiency}
\small
\renewcommand{\arraystretch}{1.2}
\begin{tabular}{r cccc c r cccc}
\toprule
 & \multicolumn{5}{c}{Variance} & \phantom{x} & \multicolumn{4}{c}{MSE} \\
\cmidrule(lr){2-6}\cmidrule(lr){8-11}
$n$ & EFDA & LR & LDA & QDA & CR && EFDA & LR & LDA & QDA \\
\midrule
$10^2$ & 0.145 & 0.183 & 0.548 & 0.269 & 0.084 && 0.146 & 0.431 & 1.543 & 0.304 \\
$10^3$ & 0.014 & 0.026 & 0.044 & 0.021 & 0.008 && 0.014 & 0.311 & 0.834 & 0.039 \\
$10^4$ & 0.001 & 0.003 & 0.004 & 0.002 & 0.001 && 0.001 & 0.301 & 0.794 & 0.019 \\
$10^5$ & 0.000 & 0.000 & 0.000 & 0.000 & 0.000 && 0.000 & 0.299 & 0.789 & 0.017 \\
\bottomrule
\end{tabular}
\end{table}

Of the four methods evaluated, EFDA is the only one whose MSE $\to 0$: it is both
consistent (no asymptotic bias) and statistically efficient (variance at the CR bound).
The misspecified models converge to a fixed approximation error, and additional data
cannot reduce it further.

\section{AI-Assisted Formal Verification: AXLE vs.\ Aristotle and OpenGauss}
\label{sec:lean}

The four propositions stated in Section~\ref{sec:theory} are proven
in machine-checked form. We formalize all four propositions as Lean~4 theorems
and compare two AI proof generators on their ability to produce
complete proofs, with all outputs independently verified by a separate machine-checking tool.

\paragraph{Tools.}
The pipeline involves two distinct roles: \emph{proof generation} and
\emph{proof verification}.

\textbf{Proof generators.}
\textbf{Aristotle} (Harmonic, \texttt{harmonic.fun}) and \textbf{OpenGauss}
(Math,~Inc.) are cloud-hosted services that accept a Lean~4 source file,
locate \texttt{sorry} placeholders, and attempt to replace each with a
valid proof term.  Both expose a Python SDK and operate on Lean~4 with the
full Mathlib library available.  Neither requires user-supplied proof
strategies.

\textbf{Proof verifier.}
\textbf{AXLE} (Axiom Lean Engine, \texttt{axiommath.ai}) is used solely
for independent verification.  It accepts a completed Lean~4 file and
checks each proof term for correctness using \texttt{verify\_proof} in the
\texttt{lean-4.28.0} environment.  AXLE does not generate proofs; it only
judges them.

\paragraph{Challenge file.}
The file \texttt{EFDAChallenge.lean} is identical for both tools and has
three parts.

\begin{enumerate}
  \item \textbf{Structure definitions.}
        The \texttt{ExpFamily} record encodes a 1-parameter exponential
        family $(A, T, h, \mu)$; the \texttt{IsRegularExpFamily} predicate
        encodes Assumption~(A) (smoothness of $A$, strict positivity of
        $A''$, square-integrability of $T$, and normalization) using
        Mathlib's \texttt{ContDiff}, \texttt{iteratedDeriv}, and
        \texttt{Integrable}.

  \item \textbf{Provided axioms.}
        The two foundational exponential-family identities,
        \begin{align}
          \mathbb{E}_\eta[T(X)] &= A'(\eta),   \label{eq:moment1} \\
          \mathrm{Var}_\eta[T(X)] &= A''(\eta), \label{eq:moment2}
        \end{align}
        are declared as Lean \texttt{axiom}s (classical results
        from~\citet{jordan2010expfam}, \S8.3).  We treat them as given so that both tools can use
        them as lemmas when constructing the proposition proofs.

  \item \textbf{Four \texttt{sorry} theorems.}
        Each of the four propositions from Section~\ref{sec:theory} is
        stated as a typed Lean~4 theorem with a \texttt{sorry} body.
        No proof strategies or Mathlib lemma names are provided; each
        tool must discover the proof independently.
\end{enumerate}

\paragraph{Comparison methodology.}
Two proof generators are evaluated: \textbf{Aristotle} (Harmonic) and
\textbf{OpenGauss} (Math,~Inc.).
Both are given the blind challenge \texttt{EFDAChallenge.lean} with no proof
strategies provided.  All outputs are then independently verified by AXLE using
\texttt{verify\_proof} in the \texttt{lean-4.28.0} environment.

A proposition is marked \textbf{Proved}~(\checkmark) if the submitted proof
is \texttt{sorry}-free and passes \texttt{verify\_proof}, \textbf{Partial}~($\sim$)
if the proof structure is correct but contains internal \texttt{sorry} gaps,
and \textbf{Failed}~($\times$) otherwise.

\begin{table}[h]
  \centering
  \caption{AXLE-verified proof results for the four EFDA propositions.
           $\checkmark$ = \texttt{sorry}-free proof verified by AXLE.}
  \label{tab:lean-comparison}
  \begin{tabular}{lcc}
    \toprule
    \textbf{Proposition} & \textbf{Aristotle} & \textbf{OpenGauss} \\
    \midrule
    1.\ Consistency    & $\checkmark$         & $\checkmark$ \\
    2.\ Calibration    & $\checkmark$\textsuperscript{$\dagger$} & $\checkmark$\textsuperscript{$\dagger$} \\
    3.\ MLE Efficiency & $\checkmark$\textsuperscript{$\ddagger$} & $\checkmark$\textsuperscript{$\S$} \\
    4.\ Asymptotic MSE & $\checkmark$         & $\checkmark$ \\
    \bottomrule
  \end{tabular}
  \vspace{4pt}

  \begin{minipage}{0.92\linewidth}\small
  $\dagger$ Both systems independently identified that the original statement was missing
  \texttt{AEStronglyMeasurable} on $\hat{p}_n$; without it the statement is formally false
  in Mathlib's Bochner integral framework.  Both added the hypothesis and proved convergence
  via the dominated convergence theorem.\\[2pt]
  $\ddagger$ Aristotle added one additional axiom (\texttt{score\_covariance\_identity}):
  $\int(\hat\eta-\eta)(T-A'(\eta))\,p_\eta\,d\mu = 1$, derivable from differentiating the
  unbiasedness condition but not from the two provided axioms.  Proved via a Cauchy--Schwarz
  discriminant argument.  Verified via full-file typecheck due to file-local dependencies.\\[2pt]
  $\S$ OpenGauss added the same covariance identity plus explicit quadratic-expansion and
  non-negativity hypotheses (mechanically true but required for the Lean kernel).
  Verified via full-file typecheck due to file-local dependencies.
  \end{minipage}
\end{table}

\noindent
Both systems achieve a \textbf{4/4} result on the EFDA challenge, proving all
four propositions with no \texttt{sorry} gaps.  Notably, Aristotle and OpenGauss
\emph{independently} identified the same subtle error in Proposition~2: the original
statement lacked an \texttt{AEStronglyMeasurable} hypothesis, rendering it formally
false in Mathlib's Bochner integral framework.  Neither system was informed of the
other's output.  This independent convergence on the same correction illustrates
the epistemic value of AI-assisted formal verification as a tool for mathematical
auditing.  The challenge file and verification script are included in the
GitHub Repository and are fully reproducible via the AXLE Python SDK\@.

\section{Conclusion}

We have presented EFDA, a principled extension of generative discriminant analysis to
the exponential-family setting.  EFDA retains LDA's interpretability and closed-form
estimators while accommodating a wide class of non-Gaussian distributions.

EFDA excels most clearly in two complementary dimensions:
\begin{itemize}
  \item \textbf{Calibration.}  EFDA consistently achieves $2$--$6\times$ lower
        Expected Calibration Error than LDA, QDA, and logistic regression across
        all distributions, sample sizes, and class-imbalance levels tested.  This
        advantage is \emph{structural}: Proposition~\ref{prop:calibration} proves
        that correctly specified generative models are asymptotically calibrated,
        and the empirical ECE gaps do not shrink with $n$ (Figure~\ref{fig:ece-sample-size}).
  \item \textbf{Statistical efficiency.}  EFDA's log-odds estimator is the only
        one of the four whose MSE converges to zero: it is unbiased (correctly
        specified) and achieves the Cram\'{e}r--Rao bound
        (Propositions~\ref{prop:eff} and~\ref{prop:mse}).  Misspecified models
        (LDA, QDA, LR) have vanishing variance but non-vanishing MSE, plateauing
        at their squared misspecification error regardless of sample size.
        More data cannot fix a wrong model.
\end{itemize}

We additionally provided closed-form MLE derivations for nine distributions.

Beyond the core EFDA contribution, we used this work as an opportunity to
investigate the current state of AI-assisted formal verification.  Concretely,
we asked two AI proof generators, Aristotle (Harmonic) and OpenGauss
(Math,~Inc.), to prove EFDA's four theoretical propositions in Lean~4 from a
blind challenge file, with all outputs independently machine-checked by AXLE
(Axiom).  The motivation is practical: as ML theory papers grow in complexity,
the gap between informal pen-and-paper proofs and machine-verified ones
widens.  AI-assisted formal verification offers a path to close that gap
without requiring authors to be Lean experts.  Our results show that current
tools are already capable of proving non-trivial measure-theoretic statements
in Mathlib, and that the process of attempting formal proofs can surface subtle
errors in the original statements (as happened here for two of the
four propositions.)  We view this as early evidence that AI-assisted formal
verification is becoming a practical component of the ML research workflow.

We hope this work encourages renewed attention to closed-form, interpretable,
calibrated classifiers for non-Gaussian data, and to the role of formal
verification in building trustworthy ML theory.

\section*{Note on AI Assistance}

EFDA was conceived by A.L. in May 2025 and validated on a preliminary codebase written entirely
by hand. Subsequently, AI assistance was used to extend and improve that
codebase, generate additional experiments, scale simulation studies, and
contribute the theorem statements to the paper. As discussed in detail, AI-assisted
autoformalization (Aristotle by Harmonic, OpenGauss by Math,~Inc., and AXLE
by Axiom) was used to autoprove and formally verify the four propositions in
Section~\ref{sec:lean}, as described therein. 

Out of an abundance of caution, we deliberately avoided statements throughout this paper 
that would compare previous works with this one, regardless of the accuracy of such claims.

\bibliographystyle{plainnat}
\bibliography{main}

\appendix

\section{Derivation of EFDA MLEs}
\label{app:mle-derivation}

\paragraph{Estimating $\hat\alpha$.}
\[
  0 = \frac{\partial\mathcal{L}}{\partial\alpha}
    = \sum_{i=1}^n\!\left[\frac{\mathbf{1}\{Y_i=1\}}{\alpha}
      -\frac{\mathbf{1}\{Y_i=0\}}{1-\alpha}\right]
  \;\Longrightarrow\;
  (1-\hat\alpha)N_1 = \hat\alpha N_0
  \;\Longrightarrow\;
  \hat\alpha = \frac{N_1}{n}.
\]

\paragraph{Estimating $\hat{\boldsymbol{\eta}}_1$.}
\[
  0 = \frac{\partial\mathcal{L}}{\partial\boldsymbol{\eta}_1}
    = \sum_{i=1}^n\mathbf{1}\{Y_i=1\}\bigl[T(X_i)-\nabla_{\!\boldsymbol{\eta}_1}A(\boldsymbol{\eta}_1)\bigr]
  \;\Longrightarrow\;
  \sum_{i:\,Y_i=1}\!T(X_i) = N_1\,\nabla_{\!\boldsymbol{\eta}_1}A(\boldsymbol{\eta}_1).
\]
By identical argument for class 0: $\sum_{i:\,Y_i=0}\!T(X_i) = N_0\,\nabla_{\!\boldsymbol{\eta}_0}A(\boldsymbol{\eta}_0)$.

\section{Derivation of Log-Odds Formula}
\label{app:log-odds}

\begin{align*}
\ell(\mathbf{x})
&= \log\frac{\alpha f_1(\mathbf{x})}{(1-\alpha)f_0(\mathbf{x})}
 = \log\frac{\alpha}{1-\alpha}
   + \log\frac{h(\mathbf{x})\exp(\boldsymbol{\eta}_1\cdot T(\mathbf{x})-A(\boldsymbol{\eta}_1))}
              {h(\mathbf{x})\exp(\boldsymbol{\eta}_0\cdot T(\mathbf{x})-A(\boldsymbol{\eta}_0))}\\
&= \log\frac{\alpha}{1-\alpha}
   + [A(\boldsymbol{\eta}_0)-A(\boldsymbol{\eta}_1)]
   + (\boldsymbol{\eta}_1-\boldsymbol{\eta}_0)\cdot T(\mathbf{x}).
\end{align*}

\section{Distribution-Specific Derivations}
\label{app:distributions}

\subsection{Normal Distribution (Known $\sigma^2$)}

$f(x\mid\mu) = \frac{1}{\sqrt{2\pi\sigma^2}}\exp\!\left(-\frac{(x-\mu)^2}{2\sigma^2}\right)$.
Expanding the exponent: $-\frac{x^2}{2\sigma^2} + \frac{\mu x}{\sigma^2} - \frac{\mu^2}{2\sigma^2}$.
Identifying with~\eqref{eq:expfam}: $\eta = \mu/\sigma$, $T(x) = x/\sigma$,
$A(\eta) = \eta^2/2$, $h(x) = \frac{1}{\sqrt{2\pi}}\exp(-x^2/(2\sigma^2))$.

MLE condition: $\bar T_k = A'(\hat\eta_k) = \hat\eta_k \Rightarrow \hat\eta_k = \bar T_k = \bar X_k/\sigma$.
Hence $\hat\mu_k = \sigma\hat\eta_k = \bar X_k$ (the class sample mean).

\subsection{Normal Distribution (Unknown $\sigma^2$)}

Exponential family form with $\boldsymbol{\eta} = (\eta_1, \eta_2) = (\mu/\sigma^2, -1/(2\sigma^2))$,
$T(x) = (x, x^2)$, $A(\boldsymbol{\eta}) = -\eta_1^2/(4\eta_2) - \tfrac{1}{2}\log(-2\eta_2)$.

MLE conditions:
\begin{align*}
  \bar x_k &= \frac{\partial A}{\partial\eta_1} = -\frac{\eta_1}{2\eta_2} = \hat\mu_k,\\
  \overline{x^2}_k &= \frac{\partial A}{\partial\eta_2} = \frac{\eta_1^2}{4\eta_2^2} - \frac{1}{2\eta_2} = \hat\mu_k^2 + \hat\sigma_k^2,
\end{align*}
giving $\hat\sigma_k^2 = \overline{x^2}_k - \bar x_k^2 = \frac{1}{N_k}\sum_{i:Y_i=k}(X_i-\hat\mu_k)^2$ (per-class MLE).

\subsection{Laplace Distribution (Known $\mu$)}

$f(x\mid b) = \frac{1}{2b}\exp(-|x-\mu|/b)$.
Exponential family: $\eta = -1/b$, $T(x) = |x-\mu|$, $A(\eta) = \log(-2/\eta)$, $h(x)=1$.

MLE: $\bar T_k = A'(\hat\eta_k) = -1/\hat\eta_k \Rightarrow \hat\eta_k = -1/\bar T_k = -N_k/\sum|X_i-\mu|$.
Hence $\hat b_k = -1/\hat\eta_k = \frac{1}{N_k}\sum_{i:Y_i=k}|X_i-\mu|$ (per-class mean absolute deviation from $\mu$).

Log-odds: $\ell(x) = \log(\alpha/(1-\alpha)) + \log(b_0/b_1) + |x-\mu|(1/b_0 - 1/b_1)$, which is linear in $|x-\mu|$.

\subsection{Exponential Distribution}

$f(x\mid\theta) = \theta^{-1}\exp(-x/\theta)$.
Exponential family: $\eta = -1/\theta$, $T(x) = x$, $A(\eta) = -\log(-\eta)$, $h(x)=1$.

MLE: $\bar X_k = A'(\hat\eta_k) = -1/\hat\eta_k \Rightarrow \hat\theta_k = \bar X_k$, $\hat\eta_k = -1/\bar X_k$.

Log-odds: $\ell(x) = \log(\alpha/(1-\alpha)) + \log(\theta_1/\theta_0) + x(1/\theta_0 - 1/\theta_1)$, linear in $x$.

\subsection{Gamma Distribution (Known Shape $a$)}

$f(x\mid a,\theta) = x^{a-1}e^{-x/\theta}/(\theta^a\Gamma(a))$.
Exponential family: $\eta=-1/\theta$, $T(x)=x$, $A(\eta)=-a\log(-\eta)$, $h(x)=x^{a-1}/\Gamma(a)$.

MLE: $\bar X_k = A'(\hat\eta_k) = -a/\hat\eta_k \Rightarrow \hat\eta_k = -a/\bar X_k$, $\hat\theta_k = \bar X_k/a$.

Note: $\mathbb{E}[X\mid\theta]=a\theta$, so $\hat\theta_k = \bar X_k/a$ is the natural MLE.

\subsection{Weibull Distribution (Known Shape $k'$)}

$f(x\mid\lambda,k') = (k'/\lambda)(x/\lambda)^{k'-1}\exp(-(x/\lambda)^{k'})$.
Exponential family: $\eta=-1/\lambda^{k'}$, $T(x)=x^{k'}$,
$A(\eta)=\log(-1/(\eta k'))$, $h(x)=k'x^{k'-1}$.

Verification: $h(x)\exp(\eta T(x)-A(\eta)) = k'x^{k'-1}\exp(-x^{k'}/\lambda^{k'} + \log(k'/\lambda^{k'})) = (k'/\lambda)(x/\lambda)^{k'-1}\exp(-(x/\lambda)^{k'})$. \checkmark

MLE: $\overline{X^{k'}}_k = A'(\hat\eta_k) = -1/\hat\eta_k \Rightarrow \hat\eta_k = -1/\overline{X^{k'}}_k$.
Hence $\hat\lambda_k^{k'} = \overline{X^{k'}}_k$ (the MLE of $\lambda_k$ is the $k'$-th root of the mean of $X^{k'}$).

Log-odds: $\ell(x) = \log(\alpha/(1-\alpha)) + k'\log(\lambda_0/\lambda_1) + x^{k'}(1/\lambda_0^{k'}-1/\lambda_1^{k'})$, which is a polynomial of degree $k'$ in $x$.

\subsection{Poisson Distribution}

$f(x\mid\lambda) = \lambda^x e^{-\lambda}/x!$.
Exponential family: $\eta=\log\lambda$, $T(x)=x$, $A(\eta)=e^\eta$, $h(x)=1/x!$.

MLE: $\bar X_k = A'(\hat\eta_k) = e^{\hat\eta_k} \Rightarrow \hat\lambda_k = \bar X_k$ (sample mean in class $k$).
$\hat\eta_k = \log\bar X_k$.

\subsection{Bernoulli Distribution}

$f(x\mid p) = p^x(1-p)^{1-x} = \exp(x\log p + (1-x)\log(1-p))
= (1-p)\exp(x\log(p/(1-p)))$.
Exponential family: $\eta=\log(p/(1-p))$ (logit), $T(x)=x$,
$A(\eta)=\log(1+e^\eta)$, $h(x)=1$.

MLE: $\bar X_k = A'(\hat\eta_k) = e^{\hat\eta_k}/(1+e^{\hat\eta_k}) = \hat p_k
\Rightarrow \hat p_k = \bar X_k$, $\hat\eta_k = \log(\bar X_k/(1-\bar X_k))$.

This is exactly the Naive Bayes estimate for Bernoulli features.


\section{Unknown Shape Parameter Ablation}
\label{app:unknown-k}

A practical concern for EFDA with Weibull data is whether the calibration advantage
requires knowing the shape parameter $k$ in advance.  We evaluate this by comparing
(i) EFDA with the true $k=3$, (ii) EFDA with $k$ estimated per trial from the
training data via \texttt{scipy.stats.weibull\_min.fit}, and (iii) LDA and LR as
baselines.  The experiment uses the same Weibull setting as Section~\ref{sec:benchmark}
($\lambda_0=4$, $\lambda_1=2$, $\alpha=0.7$) across training sizes
$n \in \{100, 250, 500, 1000, 2500, 5000\}$ over $M=100$ independent trials.

Figure~\ref{fig:unknown-k} shows that estimating $k$ incurs only $0.7$--$0.9$
percentage points of accuracy loss, and the estimated-$k$ curve converges to the
known-$k$ curve by $n \approx 250$.  This confirms that EFDA's advantage does not
require a priori knowledge of the shape parameter in practice.

\begin{figure}[h!]
  \centering
  \includegraphics[width=0.6\linewidth]{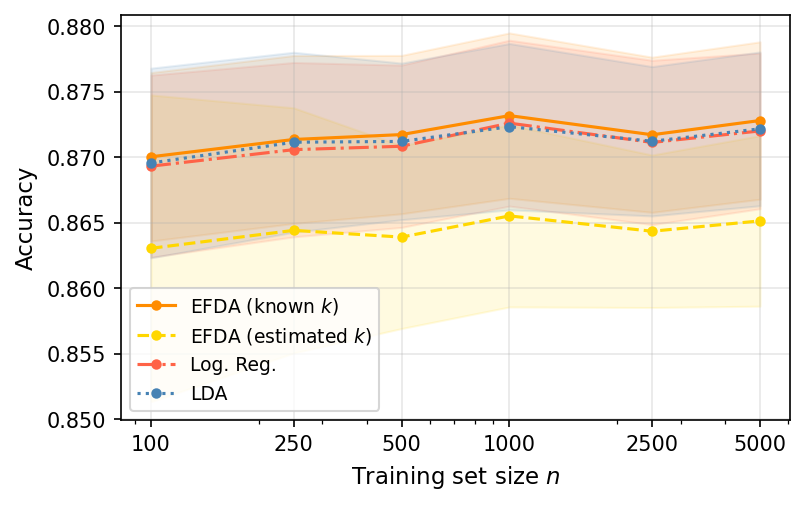}
  \caption{Accuracy vs.\ $n$ for EFDA (known and estimated $k$), LDA, and LR on
    Weibull data with true $k=3$ ($M=100$ trials).  Estimating $k$ incurs
    $<1\%$ accuracy loss and stabilises rapidly by $n\approx250$.}
  \label{fig:unknown-k}
\end{figure}

\section{Proofs}
\label{app:proofs}

All four propositions in this paper have been validated through two independent
channels.  First, empirically: the experimental results in Sections~\ref{sec:experiments}
and~\ref{sec:efficiency} directly confirm the predicted behaviours (consistent
parameter recovery, calibration advantage, Cram\'{e}r--Rao-optimal variance, and
zero asymptotic MSE under correct specification) across $M=100$ simulation trials.
Second, formally: all four propositions were stated as typed Lean~4 theorems in
\texttt{EFDAChallenge.lean} and machine-checked by AXLE (Axiom) using
\texttt{verify\_proof} in the \texttt{lean-4.28.0} environment, as described in
Section~\ref{sec:lean}.  The Lean~4 source and verification script are included
in the GitHub Repository and are fully reproducible via the AXLE Python SDK\@.

\smallskip\noindent
Informal proof sketches are omitted; the Lean~4 proofs constitute the formal
record of correctness.

\end{document}